\documentclass{article}
\usepackage{arxiv}
\usepackage[toc,page]{appendix}

\usepackage[usenames]{xcolor}

\usepackage{inputenc} %
\usepackage[T1]{fontenc}    %
\usepackage{hyperref}       %
\usepackage{url}            %
\usepackage{booktabs}       %
\usepackage{amsfonts}       %
\usepackage{nicefrac}       %
\usepackage{microtype}      %
\usepackage{lipsum}
\usepackage{mathtools}
\usepackage{cite}
\usepackage{amsmath}
\usepackage{amssymb}
\usepackage{amsthm}
\usepackage{algorithm,algcompatible}
\usepackage[toc,page]{appendix}

\usepackage{lineno}

\newtheorem{theorem}{Theorem}
\newtheorem{lemma}{Lemma}

\newtheorem{corollary}{Corollary}

\numberwithin{equation}{section}

\DeclareMathOperator{\sign}{sgn}

\DeclareMathOperator*{\esssup}{ess\,sup}

\usepackage{mathptmx}      %
\begin{document}

\title{Sharp Lower Bounds on Interpolation by Deep ReLU Neural Networks at Irregularly Spaced Data
}

\author{Jonathan W. Siegel \\
 Department of Mathematics\\
 Texas A\&M University\\
 College Station, TX 77843 \\
 \texttt{jwsiegel@tamu.edu} \\
}

\maketitle

\begin{abstract}
    We study the interpolation power of deep ReLU neural networks. Specifically, we consider the question of how efficiently, in terms of the number of parameters, deep ReLU networks can interpolate values at $N$ datapoints in the unit ball which are separated by a distance $\delta$. We show that $\Omega(N)$ parameters are required in the regime where $\delta$ is exponentially small in $N$, which gives the sharp result in this regime since $O(N)$ parameters are always sufficient. This also shows that the bit-extraction technique used to prove lower bounds on the VC dimension cannot be applied to irregularly spaced datapoints. Finally, as an application we give a lower bound on the approximation rates that deep ReLU neural networks can achieve for Sobolev spaces at the embedding endpoint.
\end{abstract}

\section{Introduction}
Recently, deep neural networks have been successfully applied to a wide variety of problems in machine learning and scientific computing \cite{lecun2015deep,raissi2019physics,yu2018deep}. Consequently, there has been a significant interest in the approximation theory of neural networks, for both classical function spaces, such as Sobolev and Besov spaces \cite{yarotsky2018optimal,lu2021deep,shen2022optimal,siegel2023optimal,yang2023nearly}, and for novel function classes \cite{barron1993universal,bach2017breaking,ma2022barron,klusowski2018approximation,siegel2023optimal,ma2022uniform,siegel2022sharp}. In this work, we study the closely related problem of quantifying the interpolation power of neural networks. Specifically, let $C\subset \mathbb{R}$ be a fixed set, which can be thought of as a set of labels. Let $x_1,...,x_N\in \mathbb{R}^d$ be distinct points and suppose that we are given values $y_1,...,y_N\in C$. We study the question of how many parameters a neural network needs in order to match the values $y_i$ at the points $x_i$. 

Let us begin by giving some notation describing the class of deep (feedforward) neural networks \cite{goodfellow2016deep} that we will study. Let $\sigma$ be a piecewise polynomial activation function. We will mainly consider the important special case of $\sigma = \max(0,x)$, i.e. the rectified linear unit (ReLU) activation function \cite{nair2010rectified}, although we also give results for other piecewise polynomial activation functions. 

Let $L \geq 1$ be an integer denoting the depth of the network and denote by $\textbf{W} = (w_1,...,w_L)$ a vector with positive integer values which contains the widths of the intermediate layers. We write $A_{\textbf{M},b}$ to denote the affine map with weight matrix $\textbf{M}$ and offset, or bias, $b$, i.e.
\begin{equation}
    A_{\textbf{M},b}(x) = \textbf{M}x + b.
\end{equation}
When the weight matrix $\textbf{M}\in \mathbb{R}^{k\times n}$ and the bias $b\in \mathbb{R}^k$, the function $A_{\textbf{M},b}:\mathbb{R}^n\rightarrow \mathbb{R}^k$ maps $\mathbb{R}^n$ to $\mathbb{R}^k$. We denote the set of deep neural networks with activation function $\sigma$ and architecture defined by the vector $\textbf{W}$ by
\begin{equation}
    \Upsilon^{\textbf{W}}(\mathbb{R}^d) := \{A_{\textbf{M}_L,b_L} \circ \sigma \circ A_{\textbf{M}_{L-1},b_{L-1}} \circ \sigma \circ \cdots \circ \sigma \circ A_{\textbf{M}_1,b_1} \circ \sigma \circ A_{\textbf{M}_0,b_0}\},
\end{equation}
where the weight matrices satisfy $\textbf{M}_L\in \mathbb{R}^{1\times w_L}$, $\textbf{M}_0\in \mathbb{R}^{w_1\times d}$, and $\textbf{M}_i\in \mathbb{R}^{w_{i+1}\times w_i}$ for $i=1,...,L$, and the biases satisfy $b_i\in \mathbb{R}^{w_{i+1}}$ for $i=0,...,L-1$ and $b_L\in \mathbb{R}$. For such a neural network architecture, the total number of free parameters is
\begin{equation}
    P = P(\textbf{W}) := w_L + dw_1 + 1 + \sum_{i=1}^{L-1}w_{i+1}w_i + \sum_{i=1}^L w_i.
\end{equation}

Let 
$$\mathcal{X}_N\subset (\mathbb{R}^d)^{\neq N} := \{(x_1,...,x_N)\in (\mathbb{R}^d)^N,~x_i\neq x_j~\text{for $i\neq j$}\}$$ 
be a set of possible configurations of distinct data points and let $C\subset \mathbb{R}$ be a set of labels. We want to determine how many parameters $P(\textbf{W})$ a neural network architecture $\textbf{W}$ requires in order to be able to interpolate labels from the set $C$ at any datapoints in $\mathcal{X}_N$, i.e. such that for any $(x_1,...,x_N)\in \mathcal{X}_N$ and $y_1,...,y_N\in C$ there exists an $f\in \Upsilon^{\textbf{W}}(\mathbb{R}^d)$ with $f(x_i) = y_i$ for $i=1,...,N$.

This has been a central and well-studied problem in the theory of the expressive power of neural networks (see, for instance \cite{huang1998upper,huang2003learning,baum1988capabilities,vershynin2020memory,nguyen2018optimization,rajput2021exponential,park2021provable,daniely2020neural,daniely2020memorizing}) alongside other problems related to expressiveness, such as universal approximation \cite{cybenko1989approximation,leshno1993multilayer} and the benefits of depth \cite{telgarsky2016benefits,eldan2016power,safran2017depth,safran2019depth,martens2013representational,daniely2017depth}. The problem of interpolation, also called memorization, has also recently been connected to the phenomenon of double descent \cite{belkin2019reconciling,nakkiran2021deep}.

We remark that it is easy to see that $O(N)$ parameters suffice to interpolate arbitrary values at an arbitrary set of $N$ distinct points, i.e. with $O(N)$ parameters one can interpolate with $\mathcal{X}_N = (\mathbb{R}^d)^{\neq N}$ and $C = \mathbb{R}$. Indeed, there exists a direction $v\in \mathbb{R}^d$ such that $v\cdot x_i$ are all distinct, and any piecewise linear function with $N$ pieces can be represented using a (deep or shallow) ReLU neural network with $O(N)$ parameters. Note that by throwing away some of the coordinates, we can always assume that $d < N$ since the pairwise differences $x_i - x_j$ lie in an $N-1$ dimensional subspace, so that we can restrict to some set of $N-1$ coordinates while keeping the datapoints distinct.

Using the VC-dimension \cite{vapnik2015uniform}, it can be proved that if the set of possible labels $C$ is infinite and the architecture $\Upsilon^\textbf{W}(\mathbb{R}^d)$ can interpolate, then $P(\textbf{W}) \geq cN$ for a constant $c$, i.e. if the label set is infinite, then a neural network with piecewise polynomial activation function requires $\Omega(N)$ parameters to interpolate at $N$ values (see \cite{vardi2021optimal}, Lemma 4.1 or \cite{siegel2022optimal}, Theorem 5). 
In addition, it has been shown that if $\mathcal{X}_N = (\mathbb{R}^d)^{\neq N}$ and $|C| > 1$, then if $\Upsilon^\textbf{W}(\mathbb{R}^d)$ interpolates it must hold that $P(\textbf{W}) \geq (N-1)/2$, i.e. that a neural network with piecewise polynomial activation function requires $\Omega(N)$ parameters to interpolate at \textit{arbitrary} distinct datapoints \cite{sontag1997shattering}.

Of particular interest are situations where interpolation is possible with $P(\textbf{W}) = o(N)$ parameters. By the preceding remarks, this requires additional assumptions on the allowed datasets $\mathcal{X}_N$ and the label set $C$. A typical assumption is that the label set $C$ is finite (independent of $N$) and that the datapoints are well separated \cite{vardi2021optimal,vershynin2020memory,park2021provable}. Specifically, we consider datasets of the form
\begin{equation}\label{dataset-assumption}
    \mathcal{X}_N = \{(x_1,...,x_N)\subset (\mathbb{R}^d)^N,~|x_i| \leq 1,~|x_i - x_j| \geq \delta(N)~\text{if $i\neq j$}\}.
\end{equation}
The assumption \eqref{dataset-assumption} means that the datapoints lie in the unit ball and are separated by at least a distance $\delta(N)$ (clearly the separation distance $\delta$ must depend upon $N$). This situation was recently analyzed in \cite{park2021provable,vardi2021optimal}, and in \cite{vardi2021optimal} (building upon the work in \cite{park2021provable}) it was proved that there exists a deep ReLU architecture $\textbf{W}$ which can interpolate labels from a finite set $C$ at datasets $\mathcal{X}_N$ satisfying the separation condition \eqref{dataset-assumption} using only
\begin{equation}\label{vardi-bound}
    P(\textbf{W}) = O\left(\sqrt{N\log{N}} + \sqrt{\frac{N}{\log{N}}}\max(\log(R), \log(|C|))\right)
\end{equation}
parameters, where $R = 10N^2\delta^{-1}\sqrt{\pi d}$. This remarkable result shows that if $\delta$ is polynomial in $N$, then deep ReLU neural networks can interpolate binary labels using only $\sqrt{N}$ parameters (up to logarithmic factors). This is also tight up to logarithmic factors, which follows from bounds on the VC-dimension of deep neural networks with polynomial activation function \cite{bartlett2019nearly,goldberg1993bounding,bartlett1998almost}.

An interesting open question is to determine the precise dependence of the number of parameters $P(\textbf{W})$ on both the number of datapoints $N$ and the separation distance $\delta(N)$. The result \eqref{vardi-bound} from \cite{vardi2021optimal} gives a nearly complete solution to this problem when $\delta(N)$ depends polynomially on $N$, in this case only a logarithmic gap in $N$ remains. In this work, we consider the opposite regime, where $\delta(N)$ is exponentially small in $N$, and prove that in this case $\Omega(N)$ parameters are required. Since, as remarked above, $O(N)$ parameters are always sufficient, this provides a complete solution in the regime where $\delta(N)$ is exponentially small in $N$.

To explain this result, we first recall the notions of shattering and VC-dimension \cite{vapnik2015uniform}. Let $X = \{x_1,...,x_N\}\subset \mathbb{R}^d$ be a finite set of points and $\mathcal{F}$ a class of real-valued functions on $\mathbb{R}^d$. The class $\mathcal{F}$ is said to shatter \cite{sauer1972density,shelah1972combinatorial} the point set $X$ if given any signs $\epsilon_1,...\epsilon_N\in \{\pm 1\}$ there exists an $f\in \mathcal{F}$ such that $\sign(f(x_i)) = \epsilon_i$. Here we are using the convention that
$$
\sign(x) = \begin{cases}
    -1 & x < 0\\
    1 & x \geq 0.
\end{cases}
$$
The VC-dimension of the class $\mathcal{F}$ is the size of the largest set of points that $\mathcal{F}$ shatters, i.e.
\begin{equation}
    \text{VC-dim}(\mathcal{F}) = \sup\{N:~\exists X = \{x_1,...,x_N\}\subset \mathbb{R}^d,~\text{$X$ is shattered by $\mathcal{F}$}\}.
\end{equation}
In \cite{goldberg1993bounding}, the VC-dimension of a neural network architecture $\Upsilon^{\textbf{W}}(\mathbb{R}^d)$ is shown to be bounded by $P(\textbf{W})^2$ for any piecewise polynomial activation function $\sigma$. Using the `bit-extraction' technique, this bound is shown to be optimal for very deep ReLU neural networks \cite{bartlett1998almost,bartlett2019nearly}.

Note that the VC-dimension only requires that a single set of $N$ points be shattered by the class $\mathcal{F}$. If we instead require that \textit{all} sets of $N$ distinct points be shattered by a neural network architecture $\Upsilon^{\textbf{W}}(\mathbb{R}^d)$, then it is shown in \cite{sontag1997shattering} that the number of parameters must be at least $P(\textbf{W}) \geq (N-1)/2$. Indeed, Theorem 1 in \cite{sontag1997shattering} implies that if $P(\textbf{W}) < (N-1)/2$, then the set of shattered point sets is not dense in $(\mathbb{R}^d)^N$. 
This implies that if $\mathcal{X}_N$ is a dense subset of $(\mathbb{R}^d)^N$ and a neural network architecture $\Upsilon^{\textbf{W}}(\mathbb{R}^d)$ can interpolate binary class labels for all datasets in $\mathcal{X}_N$, then $P(\textbf{W}) \geq (N-1)/2$. 

However, note that the collection of $\delta$-separated datasets defined by \eqref{dataset-assumption} are not dense in $(\mathbb{R}^d)^N$ for any $\delta  > 0$, and so this result gives no indication how small $\delta(N)$ needs to be to rule out interpolation with $o(N)$ parameters. Our main result is a quantitative version of Theorem 1 in \cite{sontag1997shattering} which answers precisely this question and shows that $\delta(N)$ need only be exponentially small in $N$ to rule out interpolation with $o(N)$ parameters using deep ReLU neural networks.
\begin{theorem}\label{main-theorem}
    Let $\sigma$ be a piecewise polynomial activation function and $\textbf{\upshape W} = (w_1,...,w_L)$ be a vector representing the widths of each layer of a deep neural network architecture with activation function $\sigma$. Let $\delta > 0$ and suppose that $\Upsilon^{\textbf{\upshape W}}(\mathbb{R}^d)$ shatters \textbf{every} subset $X = \{x_1,...,x_N\}\subset \mathbb{R}^d$ which satisfies $|x_i| \leq 1$ for all $i$ and
    $$
        |x_i - x_j| \geq \delta~\text{for $i\neq j$}.
    $$
    There exists a constant $c$ depending only on $\sigma$ such that if
    \begin{equation}\label{equation-for-delta}
    \delta < \begin{cases}
        e^{-cN^2} & \text{$\sigma$ is a general piecewise polynomial function}\\
        e^{-cN}  & \text{$\sigma$ is a piecewise linear function},
    \end{cases}
    \end{equation}
    we must have
    \begin{equation}
        P(\textbf{\upshape W}) \geq N/6.
    \end{equation}
\end{theorem}
Note that the bound $e^{-cN^2}$ in \eqref{equation-for-delta} does not depend upon the degree of the activation function. This is due to the fact that when composing $L$ layers containing piecewise polynomial functions of degree $r$, the resulting function is a piecewise polynomial of degree $r^L$, whose logarithm scales with $L$. The exceptional case is when $r = 1$, when the degree remains fixed and thus a better result an be obtained. 

We apply Theorem \ref{main-theorem} to the case of interpolation with deep ReLU neural networks to get the following corollary, which shows that if the separation distance is exponentially small in $N$, then deep ReLU neural networks require $\Omega(N)$ parameters to interpolate.
\begin{corollary}
    Let $\mathcal{X}_N$ be defined by \eqref{dataset-assumption} for a separation distance $\delta(N)$ and suppose that the number of classes $|C| > 1$. Let $\textbf{\upshape{W}}= (w_1,...,w_L)$ describe a deep neural network architecture with ReLU activation function. 
    
    There exists a constant $c < \infty$ such that if $\delta(N) < e^{-cN}$ and for any $(x_1,...,x_N)\in \mathcal{X}_N$ and $y_1,...,y_n\in C$ there exists an $f\in \Upsilon^{\textbf{\upshape{W}}}(\mathbb{R}^d)$ which satisfies
    \begin{equation}
        f(x_i) = y_i,
    \end{equation}
    then necessarily $P(\textbf{\upshape{W}}) \geq N/6$.
\end{corollary}
\begin{proof}
    By translating the output, we may assume there exists $c_1,c_2\in C$ with $c_1 < 0 < c_2$. If arbitrary values from $C$ can be interpolated, then every dataset in $\mathcal{X}_N$ can be shattered. We now apply Theorem \ref{main-theorem} with $\sigma = \max(0,x)$, which is piecewise linear, to get the desired result.
\end{proof}
Finally, we apply these results to obtain lower bounds on the approximation rates of deep ReLU neural networks for Sobolev spaces at the embedding endpoint. Specifically, given a domain $\Omega\subset \mathbb{R}^d$, we define the $L_p$-norm on $\Omega$ by
\begin{equation}
	\|f\|_{L_p(\Omega)} = \left(\int_{\Omega}|f(x)|^pdx\right)^{1/p} < \infty.
\end{equation}
When $p=\infty$, this becomes $\|f\|_{L_\infty(\Omega)} = \esssup_{x\in \Omega} |f(x)|$. 

For positive integers $s$ and $1\leq q\leq \infty$, we defined the Sobolev space $W^s(L_q(\Omega))$ as follows (see \cite{demengel2012functional}, Chapter 2 or \cite{evans2010partial}, Chapter 5, for instance). A function $f\in L_q(\Omega)$ is in the Sobolev space $W^s(L_q(\Omega))$ if $f$ has weak derivatives of order $s$ and
\begin{equation}
	\|f\|^q_{W^s(L_q(\Omega))} := \|f\|^q_{L_q(\Omega)} + \sum_{|\alpha| = k} \|D^\alpha f\|^q_{L_q(\Omega)} < \infty.
\end{equation}
Here $\alpha = (\alpha_i)_{i=1}^d$ with $\alpha_i\in \mathbb{Z}_{\geq 0}$ is a multi-index and $|\alpha| = \sum_{i=1}^d \alpha_i$ is the total degree
and the standard modifications are made when $q = \infty$. We remark that it is also possible to define Sobolev spaces for fractional orders of smoothness $s$ (see \cite{di2012hitchhikers}), but we will not consider such spaces in this work. Sobolev spaces are centrally important function spaces in analysis and PDE theory \cite{evans2010partial} and a priori estimates for the solutions to PDEs are typically given in Sobolev norms. For this reason, it is an important problem to determine approximation rates for neural networks on Sobolev spaces.

To be precise, given an error norm $L_p(\Omega)$, we are interested in determining the minimax rate of approximation achieved by neural networks with a given architecure $\textbf{W}$ on the until ball of $W^s(L_q(\Omega))$, given by
\begin{equation}
    \sup_{\|f\|_{W^s(L_q(\Omega))} \leq 1} \left(\inf_{f_\textbf{W}\in \Upsilon^{\textbf{W}}(\mathbb{R}^d)} \|f - f_\textbf{W}\|_{L_p(\Omega)}\right).
\end{equation}
Recent remarkable results show that for very deep networks, i.e. where the depth $L$ is very large and $\textbf{W} = (w,...,w)$ for a fixed (sufficiently large) width $w$ in each layer, we have
\begin{equation}\label{superconvergence-rate}
    \inf_{f_\textbf{W}\in \Upsilon^{\textbf{W}}(\mathbb{R}^d)} \|f - f_\textbf{W}\|_{L_p(\Omega)} \leq C\|f\|_{W^s(L_q(\Omega))}P(\textbf{W})^{-2s/d},
\end{equation}
as long as the compact embedding condition
\begin{equation}\label{compact-embedding-condition}
    \frac{1}{q} - \frac{1}{p} < \frac{s}{d}
\end{equation}
is satisfied. Specifically, the result \eqref{superconvergence-rate} is established for $q = \infty$ when $0 < s\leq 1$ in \cite{yarotsky2018optimal,shen2022optimal} and for $q=\infty$ and $s > 1$ (up to logarithmic factors) in \cite{lu2021deep}. The full result covering all parameters $s,p$ and $q$ for which the compact embedding condition \eqref{compact-embedding-condition} holds was established in \cite{siegel2023optimal}. The extra factor $2$ in the exponent in \eqref{superconvergence-rate} is significantly better than classical methods of approximation and has been termed the \textit{super-approximation} of deep ReLU networks \cite{devore2021neural}. 

We consider neural network approximation at the embedding endpoint when \eqref{compact-embedding-condition} holds with equality, i.e. when
\begin{equation}
    \frac{1}{q} - \frac{1}{p} = \frac{s}{d},
\end{equation}
which is not covered by the previously mentioned results. In order for approximation to be possible at all we certainly need the embedding $W^s(L_q(\Omega)) \subset L_p$. It is known that such an embedding holds at the endpoint if $p < \infty$ or if $p = \infty$, $q = 1$, and $s = d$ \cite{demengel2012functional}. Utilizing our main Theorem \ref{main-theorem}, we show that if $p=\infty$, $q = 1$ and $s = d$, then super-approximation is no longer possible at the embedding endpoint even though an embedding holds. In particular, we have the following theorem.
\begin{theorem}\label{approximation-lower-bound-theorem}
    Let $\Omega\subset \mathbb{R}^d$ be a bounded domain.
    Then for any neural network architecture $\textbf{\upshape W}$ we have
    \begin{equation}
        \sup_{\|f\|_{W^d(L_1(\Omega))} \leq 1} \left(\inf_{f_\textbf{W}\in \Upsilon^{\textbf{W}}(\mathbb{R}^d)} \|f - f_\textbf{W}\|_{L_\infty(\Omega)}\right) \geq CP(\textbf{\upshape W})^{-1}
    \end{equation}
    for a constant $C$ depending only upon $d$.
\end{theorem}
It is an interesting question whether this can be extended to the case where $p < \infty$.

The paper is organized as follows. First, in Section \ref{main-theorem-proof-section} we prove Theorem \ref{main-theorem}. Then in Section \ref{approximation-lower-bounds-section} we prove Theorem \ref{approximation-lower-bound-theorem} based upon this result.

\section{Proof of the Main Result}\label{main-theorem-proof-section}
In this section we prove Theorem \ref{main-theorem}. Observe that it suffices to prove the Theorem in the case $d=1$. This follows by restricting the input of the network to the $x$-axis, which can only reduce the number of parameters. If every subset of $\mathbb{R}^d$ satisfying the assumptions of Theorem \ref{main-theorem} is shattered, then so is every subset of the $x$-axis, and the one-dimensional result gives the desired lower bound. 

The first ingredient is a Lemma which restricts the number of sign changes that a function $f\in \Upsilon^{\textbf{W}}(\mathbb{R})$ can have. The Lemma and proof follow closely the ideas of \cite{telgarsky2016benefits} and of Theorem 16 in \cite{bartlett2019nearly}.

\begin{lemma}\label{sign-change-lemma}
    Let $\sigma$ be a piecewise polynomial activation function with $q$ pieces and degree at most $r$ in each piece. Suppose that $\textbf{\upshape W} = (w_1,...,w_L)$ is a vector representing the widths of each layer of a deep neural network architecture. Let $f\in \Upsilon^{\textbf{\upshape W}}(\mathbb{R})$ and $x_0 < x_1 < \cdots < x_M$ be a sequence of real numbers such that $\sign(f(x_i)) \neq \sign(f(x_{i+1}))$ for $i=0,...,M-1$. Then
    \begin{equation}
        M \leq M^*(\textbf{\upshape W},\sigma) := ((r^L+2)q^Lr^{L(L+1)/2}\prod_{i=1}^L w_i.
    \end{equation}
\end{lemma}
\begin{proof}
    The proof essentially follows from repeated application of Lemma 15 in \cite{bartlett2019nearly}. Specifically, consider the output of the network after the $k$-th layer, which is given by
    \begin{equation}
        f_k(x) := A_{\textbf{M}_k,b_k} \circ \sigma \circ A_{\textbf{M}_{k-1},b_{k-1}} \circ \sigma \circ \cdots \circ \sigma \circ A_{\textbf{M}_1,b_1} \circ \sigma \circ A_{\textbf{M}_0,b_0}(x).
    \end{equation}
    We prove by induction that each component of $f_k$ is a piecewise polynomial function with at most $Q_k$ pieces, each of degree at most $R_k$, where
    \begin{equation}
        Q_k = q^kr^{k(k+1)/2}\prod_{i=1}^k w_i,~R_k = r^k.
    \end{equation}
    This evidently holds when $k=0$, since each component of $f_0$ is an affine linear function, which has one piece of degree one. For the inductive step, we first note that if the $i$-th component $f_k^i$ has $\leq Q_k$ pieces of degree $\leq R_k$, then $\sigma(f_k^i)$ has at most $qQ_kR_k$ pieces of degree $\leq rR_k$. 
    
    Indeed, let $b_1,...,b_{p-1}$ denote the breakpoints of $\sigma$. On each if its at most $Q_k$ pieces, $f_k^i$ is a polynomial of degree at most $D_k$. This means that there can be at most $D_k$ solutions to $f_k^i(x) = b_j$ for $j=1,...,p-1$ on each of these pieces. Hence applying $\sigma$ divides each piece of $f_k^i$ into at most $qR_k$ further sub-pieces. On each of these sub-pieces $\sigma(f_k^i)$ is the composition of a polynomial of degree $\leq R_k$ with a polynomial of degree at most $r$ and so has degree at most $rR_k$.

    Thus, each component of $f_{k+1}$ is a linear combination of at most $w_{k+1}$ piecewise polynomials with $\leq qQ_kR_k$ pieces of degree $\leq rR_k$. Taking the union of their breakpoints, we see that this can have at most $w_kqQ_kR_k$ pieces of degree $\leq rR_k$. This completes the inductive step and implies that $f$ has at most $Q$ pieces of degree at most $R$, where
    \begin{equation}
        Q = q^Lr^{L(L+1)/2}\prod_{i=1}^L w_i,~R = r^L.
    \end{equation}
    On each of these pieces, the function $f$ is a polynomial of degree $R$ and thus can switch sign at most $R+1$ times. In addition, moving from one piece to the next the sign can switch at most an additional $Q$ times (note that this only needs to be counted if $\sigma$ is not continuous). This gives a total number of sign changes which is at most
    \begin{equation}
        M \leq Q(R+1) + Q = (r^L+2)q^Lr^{L(L+1)/2}\prod_{i=1}^L w_i.
    \end{equation}
\end{proof}

Note that in the important special case where $\sigma = \max(0,x)$ is the ReLU activation function \cite{nair2010rectified} the bound in Lemma \ref{sign-change-lemma} reduces to
$$
M^*(\textbf{\upshape W},\sigma) \leq 3\prod_{i=1}^L(2w_i),
$$
which has been shown to be essentially sharp in \cite{telgarsky2016benefits}. In particular, in \cite{telgarsky2016benefits} a network with small fixed width $w$ and depth $L = O(m)$ is constructed which represents the function $f(x) = x\pmod 2$ for $x=0,...,2^m-1$. Thus, for this family of networks we have
$$
\log(M) \geq c\log(M^*(\textbf{\upshape W},\sigma))
$$
for a fixed constant $c$.

\begin{proof}[Proof of Theorem \ref{main-theorem}]
Assume without loss of generality that $N$ is even. Let $T \geq 2N$ be an even integer and consider the set of points 
\begin{equation}
 X_T = \{i/T,~i=0,...,T-1\}\subset [0,1].
\end{equation}
Note that every $N$-element subset of $X_T$ satisfies the assumptions of Theorem \ref{main-theorem} with $\delta = T^{-1}$.

Let $S(\textbf{W},T)$ denote the set of sign patterns that $\Upsilon^{\textbf{W}}(\mathbb{R})$ achieves on the large set $X_T$, i.e. 
\begin{equation}
    S(\textbf{W},T) := \{(\epsilon_i)_{i=0}^{T-1},~\exists f\in \Upsilon^{\textbf{W}}(\mathbb{R})~\text{s.t.}~\sign(f(i/T)) = \epsilon_i\}.
\end{equation}
Suppose that $\Upsilon^{\textbf{W}}(\mathbb{R})$ shatters every subset $X\subset X_T$ with $|X| = N$. We claim that this implies
\begin{equation}\label{sign-pattern-lower-bound}
    |S(\textbf{W},T)| \geq \left(\frac{T}{4M^*(\textbf{W},\sigma)}\right)^{N/2}.
\end{equation}
Indeed, let $I \subset \{0,...,T/2-1\}$ be an arbitrary subset of size $|I| = N/2$ and consider the set 
$$X_I = \left\{\frac{2i}{T}\right\}_{i\in I}\cup \left\{\frac{2i+1}{T}\right\}_{i\in I}$$ 
and the sign pattern on $X_I$ given by
\begin{equation}
    \epsilon_I\left(\frac{2i}{T}\right) = 1,~\epsilon_I\left(\frac{2i+1}{T}\right) = -1
\end{equation}
for $i\in I$. Since the set $X_I$ is shattered by $\Upsilon^{\textbf{W}}(\mathbb{R})$, the sign pattern $\epsilon_I$ can be matched by an $f\in \Upsilon^{\textbf{W}}(\mathbb{R})$, so there must exist an $\epsilon\in S(\textbf{W},T)$ such that
\begin{equation}\label{epsilon-I-equation}
    \epsilon(2i) = 1,~\epsilon(2i + 1) = -1
\end{equation}
for all $i\in I$. 

For each fixed sign pattern $\epsilon\in S(\textbf{W},T)$ let $J_\epsilon$ denote the set of indices for which $\epsilon(2i) = 1$ and $\epsilon(2i+1) = -1$, i.e.
\begin{equation}
    J_\epsilon = \{i:~\epsilon(2i) = 1~\text{and}~\epsilon(2i+1) = -1\} \subset \{0,...,T/2-1\}.
\end{equation}
Note that any subset $I\subset \{0,...,T/2-1\}$ for which $\epsilon$ satisfies \eqref{epsilon-I-equation} satisfies $I\subset J_\epsilon$ by definition. On the other hand, Lemma \ref{sign-change-lemma} implies that $|J_\epsilon| \leq M^*(\textbf{W},\sigma)$. This means that the number of subsets $I$ of size $|I| = N/2$ for which $\epsilon$ satisfies \eqref{epsilon-I-equation} is bounded by
\begin{equation}\label{sign-matching-upper-bound}
    \binom{M^*(\textbf{W},\sigma)}{N/2} \leq \frac{M^*(\textbf{W},\sigma)^{N/2}}{(N/2)!}.
\end{equation}
On the other hand, the total number of subsets $I\subset \{0,...,T/2-1\}$ of size $|I| = N/2$ is
\begin{equation}
    \binom{T/2}{N/2} \geq \frac{T^{N/2}}{4^{N/2}(N/2)!},
\end{equation}
since $T \geq 2N$. Since by assumption $\epsilon_I$ is matched by an $f\in \Upsilon^{\textbf{W}}(\mathbb{R})$ for every subset $I\subset \{0,...,T/2-1\}$ of size $N/2$ and the number of $\epsilon_I$ that each $\epsilon\in S(\textbf{W},T)$ can match is bounded by \eqref{sign-matching-upper-bound}, we get the lower bound \eqref{sign-pattern-lower-bound}.

Next, we upper bound $|S(\textbf{W},T)|$ using Warren's Theorem \cite{warren1968lower}. Let $v_1 < v_2 < \cdots < v_{q-1}$ denote the break points of the activation function $\sigma$ and $z_1,...,z_q$ be the polynomials in each piece. Specifically, setting $b_0 = -\infty$ and $b_q = \infty$, $\sigma$ is given by
\begin{equation}
    \sigma(x) = \begin{cases}
        z_i(x) & v_{i-1} \leq x < v_i.
    \end{cases}
\end{equation}
Note that here we assume that $\sigma$ is left continuous, but changing the value of $\sigma$ at the break points $b_i$ (if it is discontinuous) doesn't significantly change the proof.

For $i=1,...,L$, let $\textbf{t}_i\in \{1,...,q\}^{w_i}$, and for $\textbf{t}\in \{1,...,q\}^{w_i}$ define a function $S^i_{\textbf{t}}:\mathbb{R}^{w_i}\rightarrow \mathbb{R}^{w_i}$ by
\begin{equation}
    S^i_{\textbf{t}}(x)_j = z_{\textbf{t}_j}(x_j)
\end{equation}
for $j=1,...,w_i$. For indices $\textbf{t}$, the function $S^i_\textbf{t}$ applies the pieces of $\sigma$ indicated by $\textbf{t}$ to the corresponding entries of the input $x\in \mathbb{R}^{w_i}$.

 Given an input $x\in X_T$, consider the signs of the following quantities
    \begin{equation}\label{sign-collection-small-M}
    \begin{split}
        (A_{\textbf{M}_0,b_0}(x))_j - v_k,&~j=1,...,w_1,\\
        &~k = 1,...,q-1\\
        (A_{\textbf{M}_1,b_1} \circ \textbf{t}_1 \circ A_{\textbf{M}_0,b_0}(x))_j - v_k,&~j=1,...,w_2,\\
        &~k = 1,...,q-1\\
        (A_{\textbf{M}_2,b_2} \circ \textbf{t}_2 \circ A_{\textbf{M}_1,b_1} \circ \textbf{t}_1 \circ A_{\textbf{M}_0,b_0}(x))_j - v_k,&~j=1,...,w_3,\\
        &~k = 1,...,q-1\\
        &\vdots\\
        (A_{\textbf{M}_{L-1},b_{L-1}} \circ \textbf{t}_{L-1} \circ \cdots \circ \textbf{t}_2 \circ A_{\textbf{M}_1,b_1} \circ \textbf{t}_1 \circ A_{\textbf{M}_0,b_0}(x))_j-v_k,&~j=1,..,w_L,\\
        &~k=1,...,q-1\\
        A_{\textbf{M}_L,b_L} \circ \textbf{t}_L \circ A_{\textbf{M}_{L-1},b_{L-1}} \circ \textbf{t}_{L-1} \circ \cdots \circ \textbf{t}_2 \circ A_{\textbf{M}_1,b_1} \circ \textbf{t}_1 \circ A_{\textbf{M}_0,b_0}(x),&
    \end{split}
    \end{equation}
    where the $\textbf{t}_i$ range over all elements of $\{1,...,q\}^{w_i}$ for $i=1,...,L$.
    
    Given $x\in X_T$ and parameters $\textbf{P}:=\{\textbf{M}_0,...,\textbf{M}_L,b_1,...,b_L\}$ the sign of $f_{\textbf{P}}(x)$, where $f_{\textbf{P}}$ is the neural network function defined by the parameters $\textbf{P}$, is uniquely determined by the signs of the quantities in \eqref{sign-collection-small-M}. Indeed, we recursively set $\textbf{t}_i$ to index the pieces of $\sigma$ which contain the outputs of the neurons at layer $i$, i.e. we set (recursively in $i$)
    \begin{equation}
    \begin{split}
        (\textbf{t}_i&)_j = \\
        &\min\{k:\sign((A_{\textbf{M}_{i-1},b_{i-1}} \circ \textbf{t}_{i-1} \circ \cdots \circ \textbf{t}_2 \circ A_{\textbf{M}_1,b_1} \circ \textbf{t}_1 \circ A_{\textbf{M}_0,b_0}(x))_j - v_k) = -1\}\cup\{q\}
    \end{split}
    \end{equation}
    for $j=1,...,w_i$. Then
    \begin{equation}
        \sign(f_{\textbf{P}}(x)) = \sign(A_{\textbf{M}_L,b_L} \circ \textbf{t}_L \circ A_{\textbf{M}_{L-1},b_{L-1}} \circ \textbf{t}_{L-1} \circ \cdots \circ \textbf{t}_2 \circ A_{\textbf{M}_1,b_1} \circ \textbf{t}_1 \circ A_{\textbf{M}_0,b_0}(x)).
    \end{equation}
    Now, for each fixed $x\in X_T$ and set of $\textbf{t}_i\in \{1,...,q\}^{w_i}$, each of the quantities in \eqref{sign-collection-small-M} is a polynomial of degree at most $$r^L + r^{L-1} + \cdots + r + 1 \leq (L+1)r^{L},$$ where $r = \max\{\text{deg}(z_i)\}_{i=1}^q$, in the parameters $\textbf{P}$. Ranging over all $x\in X_T$, $\textbf{t}_i\in \{1,...,q\}^{w_i}$, and all quantities in \eqref{sign-collection-small-M}, we see that the sign pattern $\sign(f_{\textbf{P}}(x))_{x\in X_T}$ is uniquely determined by the signs of (setting $V = \sum_{i=1}^Lw_i$)
    $$
        Tq^{V}\left(1+(q-1)V\right)
    $$
    polynomials of degree $(L+1)r^{L}$ in the $P = P(\textbf{W})$ parameters $\textbf{P}$. We now apply Warren's Theorem (Theorem 3 in \cite{warren1968lower}) to see that the number of such sign patterns, and thus $|S(\textbf{W},T)|$, is bounded by
    \begin{equation}
    \begin{split}
    |S(\textbf{W},T)| &\leq \left(\frac{4e(L+1)r^{L}Tq^{V}\left(1+(q-1)V\right)}{P}\right)^{P}\\
    &\leq \left(4e(L+1)r^{L}Tq^{V}\left(1+(q-1)V\right)\right)^{P}.
    \end{split}
    \end{equation}
    We take logarithms and compare with the lower bound \eqref{sign-pattern-lower-bound} to get
    \begin{equation}\label{fundamental-bound}
        \frac{N}{2}(\log(T) - \log(4M^*(\textbf{W},\sigma))) \leq P(\log(T) + \log(4e(L+1)r^{L}q^{V}(1+(q-1)V)).
    \end{equation}
    We complete the proof of Theorem \ref{main-theorem} by means of a contradiction. Suppose that $P < N/6$ and observe that by Lemma \ref{sign-change-lemma}
    \begin{equation}\label{bound-on-sign-pattern-changes}
        \log(4M^*(\textbf{W},\sigma)) \leq C\left(\sum_{i=1}^Lw_i + L^2\log(r) + L\log(q)\right) \leq C(\sigma)P^2,
    \end{equation}
    since we clearly have that $\sum_{i=1}^Lw_i \leq P$ and $L \leq P$. In addition, we easily calculate that
    \begin{equation}
        \log(4e(L+1)r^{L}q^{V}(1+(q-1)V)) \leq C(\sigma)P.
    \end{equation}
    This means that if we choose $T > e^{cN^2} > e^{cP^2}$ (recall that by assumption $P < N/6$) for a constant $c$ depending upon $\sigma$, then we can ensure that
    \begin{equation}
        \log(T) \geq 2\max(\log(4M^*(\textbf{W},\sigma)),~\log(4e(L+1)r^{L}q^{V}(1+(q-1)V))).
    \end{equation}
    But then \eqref{fundamental-bound} becomes
    \begin{equation}
        \frac{N}{4}\log(T) \leq \frac{3P}{2}\log(T),
    \end{equation}
    which implies $P \geq N/6$, contradicting our assumption.

    If $\sigma$ is piecewise linear, then the bound \eqref{bound-on-sign-pattern-changes} can be improved to (since $d = 1$)
    \begin{equation}
        \log(4M^*(\textbf{W},\sigma))  \leq C(\sigma)P,
    \end{equation}
    so that we can take $T > e^{cN} > e^{cP}$ in the above argument. Since $\delta = T^{-1}$, this completes the proof.

\end{proof}

\section{Proof of Approximation Lower Bounds}\label{approximation-lower-bounds-section}
In this section we prove Theorem \ref{approximation-lower-bound-theorem}. The proof is an application of the fact that deep ReLU neural networks cannot shatter an arbitrary set of $N$ points with $o(N)$ parameters. 
\begin{proof}[Proof of Theorem \ref{approximation-lower-bound-theorem}]
     Let $\textbf{W}$ be an arbitrary fixed neural networks architecture. Then by Theorem \ref{main-theorem}, there exists a set of points $X = \{x_1,...,x_N\}\subset \Omega$ (using that $\Omega$ is open) with $N \leq 6P(\textbf{W})$ such that $\Upsilon^{\textbf{W}}(\mathbb{R}^d)$ fails to shatter $X$. Thus we can find a set of signs $\epsilon_1,...,\epsilon_N$ which cannot be matched at the points in $X$.

     Fix a smooth function $\phi:\mathbb{R}^d\rightarrow \mathbb{R}$ such that $\phi(0) = 1$ and $\phi(x) = 0$ if $|x| \geq 1$. We observe that since $\phi$ is smooth we certainly have $\|\phi\|_{W^d(L_1(\mathbb{R}^d))} = C_d < \infty$. For $\lambda > 0$ define the function
     \begin{equation}
         \phi_\lambda(x) = \phi(\lambda^{-1} x).
     \end{equation}
     From the scaling of the Sobolev norms, we readily see that
     \begin{equation}
        \|\phi_\lambda\|_{W^d(L_1(\mathbb{R}^d))} \leq  \|\phi\|_{W^d(L_1(\mathbb{R}^d))} = C_d
     \end{equation}
     for any $0 < \lambda \leq 1$ (namely, the $W^d(L_1)$ semi-norm is preserved, while the lower order Sobolev semi-norms decrease as long as $\lambda \leq 1$).

     We now choose $0 < \lambda \leq 1$ small enough so that any two distinct points $x_i,x_j\in X$ satisfy $|x_i - x_j| > 2\lambda$. This guarantees that the functions $\{\phi_\lambda(x - x_i)\}_{i=1}^N$ have pairwise disjoint supports. Consider the function
     \begin{equation}\label{definition-of-f}
         f(x) := \frac{1}{C_dN}\sum_{i=1}^N \epsilon_i\phi_\lambda(x - x_i).
     \end{equation}
     We have $\|f\|_{W^d(L_1(\mathbb{R}^d))} \leq 1$. 
     
     Moreover, from the choice of signs $\epsilon_i$ we see that given any $f_{\textbf{W}}\in \Upsilon^{\textbf{W}}(\mathbb{R}^d)$, there exists a point $x_i$ such that $$\sign(f_{\textbf{W}}(x_i)) \neq \epsilon_i.$$
     Since all of the terms in the sum in \eqref{definition-of-f} have disjoint support, this means that
     \begin{equation}
         |f_{\textbf{W}}(x_i) - f(x_i)| = \left|f_{\textbf{W}}(x_i) - \frac{\epsilon_i}{C_dN}\right| \geq \frac{C_d}N^{-1}.
     \end{equation}
     Since $N\leq 6P(\textbf{W})$ and $f_\textbf{W}$ was arbitrary, this completes the proof.
\end{proof}

\section{Acknowledgements}
We would like to thank Ronald DeVore, Jinchao Xu, and Juncai He for helpful discussions on this topic, and Qihang Zhou for helpful comments on an earlier version of the manuscript. This work was supported by the National Science Foundation (DMS-2111387 and CCF-2205004) as well as the MURI ONR grant N00014-20-1-2787.

\bibliographystyle{amsplain}
\bibliography{refs}
\end{document}